\documentclass[accepted]{uai2026} 
\usepackage[american]{babel}

\usepackage{natbib} 
\usepackage{mathtools} 
\usepackage{booktabs} 
\usepackage{tikz} 
\usepackage[capitalize]{cleveref}
\usepackage[textwidth=1.5cm]{todonotes}
\usepackage{algorithm}
\usepackage{algorithmic}
\usepackage{amsfonts}
\usepackage{amsthm}
\usepackage{thmtools}
\usepackage{bm}

\definecolor{TueAIdark}{RGB}{56,56,56}
\definecolor{TueAIgray}{RGB}{246,246,246}
\definecolor{TueAIdarkblue}{RGB}{26,58,91}
\definecolor{TueAIaccent}{RGB}{234,75,46}
\definecolor{TueAIlightblue}{RGB}{133,203,210}
\definecolor{TueAIoceanblue}{RGB}{119,221,204}
\definecolor{TueAIoceangreen}{RGB}{119,221,159}
\definecolor{TueAIspringgreen}{RGB}{186,213,72}
\definecolor{TueAIbrightyellow}{RGB}{255,221,0}
\hypersetup{colorlinks=true,
  citecolor=TueAIdarkblue,
  urlcolor=TueAIoceanblue,
  linkcolor=TueAIaccent}

\declaretheorem[Refname={Theorem,Theorems}]{theorem}

\declaretheorem[name=Corollary,numberlike=theorem,Refname={Corollary,Corollaries}]{corollary}

\newcommand{\R}{\mathbb{R}}
\newcommand{\Rd}{\mathbb{R}^d}

\newcommand{\Normal}{\mathcal{N}}

\newcommand{\GP}{\mathcal{GP}}
\newcommand{\Kmat}{\mathbf{K}}            
\newcommand{\kvec}{\mathbf{k}}            
\newcommand{\zvec}{\mathbf{z}}            
\newcommand{\yvec}{\mathbf{y}}            
\newcommand{\xvec}{\mathbf{x}}            
\newcommand{\Ihat}{\hat{I}}               
\newcommand{\sighat}{\hat{\sigma}}         

\newcommand{\domain}{\Omega}

\newcommand{\lscale}{\ell}
\newcommand{\outscale}{\sigma^2_f}

\title{Hierarchical Bayesian Quadrature}

\author[1]{\href{mailto:tim.weiland@uni-tuebingen.de?Subject=Your UAI 2026 paper}{Tim~Weiland}{}}
\author[2]{Toni~Karvonen}
\author[1]{Philipp~Hennig}
\affil[1]{%
    Tübingen AI Center\\
    University of Tübingen\\
    Tübingen, Germany
}
\affil[2]{%
    School of Engineering Sciences\\
    Lappeenranta--Lahti University of Technology LUT\\
    Lappeenranta, Finland
}

\begin{document}
\maketitle

\begin{abstract}
  Numerical integration is a cornerstone of various scientific computing applications, such as engineering simulations and model evidence computations in probabilistic machine learning.
  Bayesian Quadrature uses Gaussian process surrogates that explicitly encode structural assumptions about the integrand to obtain integral estimates with quantified uncertainty.
  These surrogates are predominantly based on stationary covariance functions, which results in model misspecification for integrands exhibiting nonstationary behavior.
  We tackle this issue through an adaptively growing, tree-based partition of the integration domain into local stationary models.
  Our method recombines the local integral estimates through a hierarchy of GP conditioning that reintroduces cross-subdomain correlations, while model selection criteria control the tree growth to avoid unnecessary partitioning.
  The resulting algorithm is simple, requires no MCMC, and adapts its evaluation budget to local integrand complexity.
  On benchmark integration problems and a model evidence computation for an epidemiological model, Hierarchical Bayesian Quadrature achieves substantial gains over standard Bayesian Quadrature on nonstationary integrands while matching its performance on stationary ones.
\end{abstract}

\section{Introduction}\label{sec:intro}
Many tasks in fields like scientific computing and probabilistic machine learning require the computation of a definite integral of a function $f$ over some domain $\domain \subset \Rd$:
\begin{equation}
  I = \int_\Omega f(\bm{x}) d\bm{x} .
  \label{eq:definite-integral}
\end{equation}
For example, the integral could represent a model evidence or the normalizing constant in Bayes' formula~\citep{osborne2012, Adachi2022}, or summary statistics for a computer simulation~\citep{li2023}.
In practical applications, closed-form expressions for \cref{eq:definite-integral} are often either unavailable or too expensive to evaluate.

\textbf{Numerical integration} is the task of computing an approximation to \cref{eq:definite-integral}.
Such methods typically consider the integrand $f$ to be a black box by only assuming access to function evaluations $f(\bm{x})$ for $\bm{x} \in \Omega$.

There are various classes of algorithms for numerical integration that mainly differ in the extent to which they make prior assumptions about the integrand.
Monte Carlo and quasi-Monte Carlo methods \citep{evansApproximatingIntegralsMonte2000, DickKuoSloan2013} make minimal assumptions about the integrand and are particularly prevalent for high-dimensional integration.
In contrast, classical quadrature rules \citep{Davis1984}, such as Gaussian quadratures, exploit integrand smoothness to achieve faster convergence rates.

\textbf{Bayesian quadrature} (BQ) \citep{ohagan1991, mahsereci2026} approaches numerical integration through a Gaussian process (GP) surrogate model of the integrand.
Conditioning on point evaluations $f(x)$ yields a GP posterior.
As Gaussians are closed under linear operations, the integral of the posterior is a Gaussian random variable.
Thus, BQ naturally produces not only a point estimate for the integral, but also a calibrated error estimate.
For suitable choices of the kernel function of the GP, the associated computations may be performed in closed-form~\citep{briol2025}.
The ability of BQ to flexibly express assumptions such as smoothness and periodicity about the integrand through the GP prior is a major strength of the method.
Well-specified priors yield faster convergence, well-calibrated uncertainty estimates, and are more robust~\citep{Briol2019}.

\begin{figure*}[t]
  \centering
  \includegraphics[width=\linewidth]{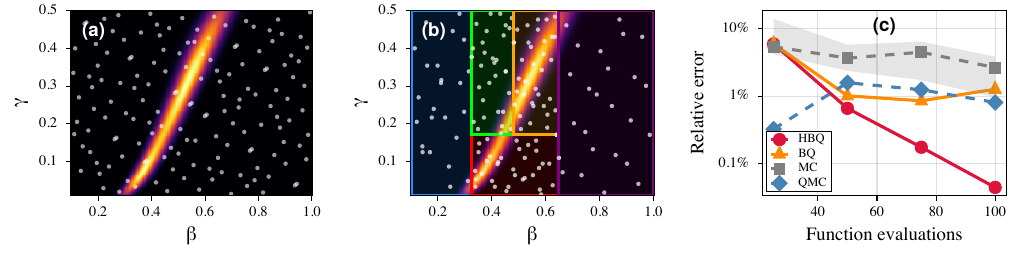}
  \caption{
    \textbf{Adaptive budget allocation for model evidence computation.}
    We consider an SIR epidemiological model with infection rate $\beta$ and recovery rate $\gamma$; each evaluation requires solving an ODE system.
    The integrand (normalized likelihood times prior) exhibits a curved ridge due to the $\beta / \gamma$ correlation.
    \textbf{(a)}~BQ with $N\!=\!150$ Sobol points: uniform coverage regardless of local complexity.
    \textbf{(b)}~HBQ with $N\!=\!150$: adaptive partitioning concentrates evaluations on the ridge.
    \textbf{(c)}~Relative error in $\log Z$; MC shows median and interquartile range over 20 runs.
  }
  \label{fig:sir-motivation}
\end{figure*}

However, BQ methods predominantly employ stationary covariance functions.
This results in \textbf{model misspecification} for integrands best described by nonstationary processes whose variability depends on location.
While BQ is known to exhibit certain robustness to smoothness misspecification~\citep{Kanagawa2020, mahsereci2026}, nonstationarity remains an unsolved challenge.
Nonstationary integrands arise naturally for example when computing posterior expectations over peaked likelihood surfaces, as well as in engineering simulations~\citep{gramacy2008}. \Cref{fig:sir-motivation} illustrates this on a model evidence computation for an SIR epidemiological model, where the integrand exhibits a sharp ridge surrounded by flat regions.
A single stationary kernel must then compromise between resolving the ridge and efficiently covering flat regions, causing inefficient allocation of the evaluation budget.

\textbf{Contributions:} 
We present \textbf{Hierarchical Bayesian Quadrature} (\textbf{HBQ}), a BQ method that automatically detects nonstationary integrand behavior and adaptively invests its compute budget accordingly.
We achieve this through a tree-based partition of the domain where the leaf surrogates are stationary GPs with varying lengthscale and smoothness.
We present an algorithm to adaptively and robustly grow such a tree on-the-fly via marginal likelihood optimization and model selection techniques.

Our method requires no MCMC and introduces little computational overhead; in fact, domain splitting can accelerate GP inference by reducing the size of each local Gram matrix.
On integrands exhibiting nonstationary behavior, HBQ achieves substantial performance gains over standard BQ, while matching its performance on stationary integrands.
Indeed, our method is closely related to the Genz--Malik rule, a state-of-the-art deterministic adaptive quadrature method.

\section{Background}\label{sec:background}

\subsection{Bayesian Quadrature}\label{sec:bq}
Bayesian Quadrature (BQ) models the integrand $f$ through a Gaussian process (GP) prior with mean function $\mu$ and kernel function $k$: $f \sim \mathcal{GP}(\mu, k)$.
Without loss of generality, we assume $\mu = 0$.
Through its smoothness and other properties the kernel function encodes assumptions about the function $f$.
See \citet{RasmussenWilliams2006} and \citet{mahsereci2026} for introductions to GPs and BQ, respectively.

We collect evaluations $\yvec$ of the integrand at locations $\bm{X} \coloneqq \left( \bm{x}_1, \dots, \bm{x}_N \right)$, where $\bm{x}_i \in \domain$.
The posterior under the evaluations is again a GP: $f \mid \left( f(\bm{X}) = \yvec \right) \sim \GP(\mu^*, k^*)$ with
\begin{align}
  \mu^*(\bm{x}) &= \kvec(x, \bm{X}) \Kmat^{-1} \yvec, \label{eq:gp-post-mean} \\
  k^*(\bm{x}, \bm{x}') &= k(\bm{x}, \bm{x}') - \kvec(\bm{x}, \bm{X})^\top \Kmat^{-1} \kvec(\bm{X}, \bm{x}'), \label{eq:gp-post-cov}
\end{align}
where $\Kmat \coloneqq [k(\bm{x}_i, \bm{x}_j)]_{i,j=1}^{N}$ is the Gram matrix and $\kvec(\bm{x}, \bm{X}) \coloneqq \kvec(\bm{X}, \bm{x}) = [k(\bm{x}, \bm{x}_i)]_{i=1}^N$.

Bounded linear transformations of Gaussian processes remain Gaussian with closed-form mean and covariance \citep{pfoertner2024}.
Consequently, we have
\begin{equation}
  \int_\domain f(\bm{x}) \, d\bm{x} \;\Big|\; \left( f(\bm{X}) = \yvec \right) \;\sim\; \Normal\!\left(\Ihat,\, \sighat^2\right), \label{eq:bq-posterior}
\end{equation}
with
\begin{align}
  \Ihat &= \zvec^\top \Kmat^{-1} \yvec, \label{eq:bq-mean} \\
  \sighat^2 &= \int_\domain \int_\domain k(\bm{x}, \bm{x}') \, d\bm{x} \, d\bm{x}' - \zvec^\top \Kmat^{-1} \zvec, \label{eq:bq-var}
\end{align}
where $\zvec \coloneqq [ \int_\domain k(\bm{x}, \bm{x_i}) \, d\bm{x} ]_{i=1}^{N}$ is the vector of kernel mean embeddings.

For suitable choices of kernel functions (see \citet{briol2025}), we can compute the quantities in \cref{eq:bq-mean,eq:bq-var} --- and thus the integral of the posterior --- in closed-form.
We obtain a point estimate $\Ihat$ of the true integral \cref{eq:definite-integral} with quantified uncertainty $\sighat^2$ that shrinks as the number of point evaluations $N$ grows.
In this sense, BQ is a prototypical probabilistic numerical method.
These are methods that cast numerical tasks as learning problems to be solved with the tools of Bayesian inference~\citep{hennig2015,pnbook22}.

\subsection{Hyperparameter Selection}\label{sec:hyperparams}
In order to obtain both an accurate point estimate and calibrated uncertainty, we require a well-specified prior.
We consider a stationary kernel function of the form
\begin{equation}
  k(\bm{x}, \bm{x'}) = \outscale \, \phi_\nu\!\left( \frac{\|\bm{x} - \bm{x'}\|}{\lscale} \right), \label{eq:kernel}
\end{equation}
where $\outscale$ is the output scale, $\lscale > 0$ is the lengthscale, and $\phi_\nu$ is a positive definite function parameterized by a smoothness parameter $\nu$.
In this work, we focus in particular on the Matérn family.
A Matérn-$\nu$ kernel is defined by
\begin{equation}\label{eq:matern}
  \phi_\nu(r) = \frac{2^{1-\nu}}{\Gamma(\nu)} \left(\sqrt{2\nu}\, r\right)^\nu K_\nu\!\left(\sqrt{2\nu}\, r\right),
\end{equation}
where $K_\nu$ is the modified Bessel function of the 2nd kind.

The hyperparameters $\theta = (\outscale, \lscale, \nu)$ are typically selected by maximizing the marginal likelihood $p(\yvec \mid \bm{X}, \theta)$.
The output scale $\outscale$ can be concentrated out analytically, leaving $\lscale$ (and optionally $\nu$) as the only parameters to optimize.

The lengthscale $\lscale$ controls the assumed regularity of the integrand: small $\lscale$ allows rapid variation, while large $\lscale$ implies slow variation. The smoothness parameter $\nu$ controls the differentiability of the sample paths.
For a nonstationary integrand, any single choice of $\lscale$ and $\nu$ is a compromise --- too smooth in some regions, too rough in others --- resulting in the model misspecification discussed in \cref{sec:intro}.

\section{Hierarchical Bayesian Quadrature}\label{sec:method}

\subsection{Domain Decomposition and Local BQ}\label{sec:decomposition}
\textbf{Nonstationary kernels.}
The natural response to the aforementioned misspecification would be a nonstationary kernel, such as the ones proposed by \citet{gibbs1997,paciorek2003}, and \citet{remes2017}.
Unfortunately, closed-form kernel mean embeddings are generally unavailable for nonstationary kernels, requiring costly approximation and negating much of BQ's computational appeal.

Our pragmatic response is to express nonstationary behavior through a hierarchy of stationary models on subdomains.
In doing so, we adapt to nonstationarity without sacrificing the benefits of stationary models.

\textbf{Domain decomposition.}
Assume we have a partition of the domain $\Omega$ into $K$ non-overlapping subdomains $\Omega_i$:
\begin{equation}
  \Omega = \cup_{i=1}^{K} \Omega_i,
\end{equation}
where the subdomains $\Omega_i$ are chosen to encapsulate heterogeneous variation and smoothness properties of the integrand.
We can then fit a local stationary GP $f_i \sim \GP(0, k_i)$ to each subdomain and use BQ to obtain integral estimates
\begin{equation}
  \int_{\Omega_i} f_i(\bm{x}) d\bm{x} \sim \Normal(\Ihat_i, \sighat_i) .
  \label{eq:leaf-integral-estimate}
\end{equation}
To keep all computations closed-form, we impose a few restrictions on the kernel class and the domain decomposition.

\textbf{Kernel choice.}
While HBQ can in principle use any kernels, we assume a half-integer Mat\'ern or Wendland kernel as a base kernel.
Both are of the form in \cref{eq:kernel} and adapt to function behavior via their lengthscale and smoothness parameters.
Kernel mean embeddings over intervals $[a, b]$ are available in closed form \citep{briol2025}.
For $d > 1$, we form a tensor product kernel $k(\bm{x}, \bm{x}') = \prod_{j=1}^{d} k_j(x_j, x_j')$.

\textbf{Domain choice.}
We further restrict the subdomains $\domain_i$ to axis-aligned hyperrectangles, i.e.\ products of intervals.
Due to the tensor product structure, the kernel mean embeddings are then closed-form as products of one-dimensional kernel mean embeddings over intervals.
Starting from the hyperrectangle $\domain$, we partition via axis-aligned binary splits.
This naturally produces hyperrectangular subdomains.

\begin{figure}[t]
  \centering
  \includegraphics[width=\columnwidth]{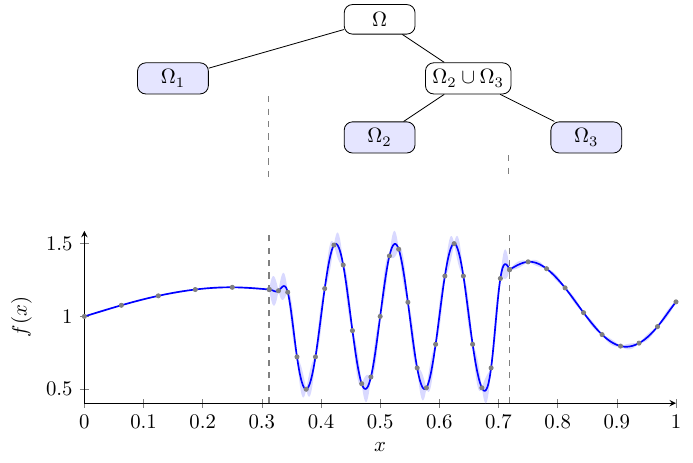}
  \caption{HBQ decomposes the domain into subdomains via a tree of axis-aligned splits. Each leaf fits a local stationary GP with its own lengthscale, adapting to local integrand complexity. Shaded regions show posterior uncertainty.}
  \label{fig:hbq-decomposition}
\end{figure}

  The resulting partition forms a binary tree (\cref{fig:hbq-decomposition}). What remains is to recombine leaf integral estimates into a global estimate and
  to decide adaptively when and where to split.

\subsection{Integral Recombination}\label{sec:recombination}
Given the leaf integral estimates from \cref{eq:leaf-integral-estimate}, how do we obtain the global integral?
Clearly, the global integral is
\begin{equation}
  \int_\domain f(\bm{x}) d\bm{x} = \sum_{i=1}^K \int_{\domain_i} f(\bm{x}) d\bm{x}.
  \label{eq:global-as-sum-of-local}
\end{equation}

The simplest approach is to assume independence of the local GPs.
Then \cref{eq:global-as-sum-of-local} is a sum of independent Gaussian random variables, and we obtain
\begin{equation}
  \int_\domain f(\bm{x}) \, d\bm{x} \;\Big|\; \left( f(\bm{X}) = \bm{y} \right) \sim \Normal\!\left(\sum_{i=1}^K \Ihat_i, \sum_{i=1}^K \sighat_i^2\right).
  \label{eq:independent-sum}
\end{equation}
We refer to this as the \emph{independent sum} estimate.
It is simple and introduces no overhead, but the underlying independence assumption is likely to be wrong in practice.

\textbf{Tree conditioning.}
To alleviate this issue, we exploit the tree structure of the partition.
When a leaf node is split, it becomes a branch node with children.
Rather than discarding the GP that was fitted to the leaf before the split, we retain it as the branch node's prior $k_{\text{p}}$.
Since this kernel was estimated on the full parent domain $\domain_{\text{p}} = \cup_{j=1}^{m} \domain_{\text{p},j}$, it encodes correlations across the split boundary.

Under this prior, the sub-integrals $I_j = \int_{\domain_{\text{p},j}} f(\bm{x}) \, d\bm{x}$ are jointly Gaussian with covariance
\begin{equation}
  C_{ij} = \int_{\domain_{\text{p},i}} \int_{\domain_{\text{p},j}} k_{\text{p}}(\bm{x}, \bm{x}') \, d\bm{x} \, d\bm{x}'.
  \label{eq:cross-covariance}
\end{equation}
Each child provides an integral estimate $(\Ihat_j, \sighat_j^2)$ from its own local GP.
We treat these as noisy observations of the linear functionals \smash{$L_j[f] = \int_{\domain_{\text{p},j}} f(\bm{x}) \, d\bm{x}$}, with observation noise $\sighat_j^2$.
Define the cross-covariance between the total integral and each sub-integral as
\begin{equation}
  c_j = \int_{\domain_{\text{p}}} \int_{\domain_{\text{p},j}} k_{\text{p}}(\bm{x}, \bm{x}') \, d\bm{x} \, d\bm{x}' = \sum_{i=1}^{m} C_{ij},
  \label{eq:cross-cov-total}
\end{equation}
and let $\bm{\Sigma} = \operatorname{diag}(\sighat_1^2, \dots, \sighat_m^2)$.
Conditioning the parent GP on the children's estimates gives
\begin{equation}
  \int_{\domain_{\text{p}}} f(\bm{x}) \, d\bm{x} \;\Big|\; \left( L_j[f] \approx \Ihat_j \right)_{j=1}^m \;\sim\; \Normal\!\left( \Ihat_{\text{p}},\, \sighat_{\text{p}}^2 \right),
  \label{eq:tree-conditioning}
\end{equation}
with
\begin{align}
  \Ihat_{\text{p}} &= \bm{c}^\top (\bm{C} + \bm{\Sigma})^{-1} \hat{\bm{I}}, \label{eq:tree-cond-mean} \\
  \sighat_{\text{p}}^2 &= \int_{\domain_{\text{p}}} \int_{\domain_{\text{p}}} k_{\text{p}}(\bm{x}, \bm{x}') \, d\bm{x} \, d\bm{x}' - \bm{c}^\top (\bm{C} + \bm{\Sigma})^{-1} \bm{c}, \label{eq:tree-cond-var}
\end{align}
where $\hat{\bm{I}} = (\Ihat_1, \dots, \Ihat_m)^\top$.
This is analogous to \cref{eq:bq-mean,eq:bq-var}, with the children's integral estimates playing the role of observations and $\bm{c}$ replacing the kernel mean embedding vector $\zvec$.

If a child is itself a branch, its estimate $(\Ihat_j, \sighat_j^2)$ is obtained by applying \cref{eq:tree-cond-mean,eq:tree-cond-var} recursively.
The global integral is thus computed bottom-up: leaf estimates propagate through the tree, with each branch reintroducing cross-subdomain correlations via its retained prior.

This construction also yields a multi-scale decomposition: the parent kernel $k_{\text{p}}$ captures coarse-scale structure, while the children's kernels model finer-scale variation within each subdomain.

\subsection{Convergence Analysis}\label{sec:convergence}
We next provide some theory for tree conditioning.
The following theorem states that the child with largest error controls the error of its parent.
Note that this theorem holds trivially also for~\eqref{eq:independent-sum}.

\begin{theorem} \label{thm:tree-error}
  Suppose that $\lvert I_j - \Ihat_j \rvert \leq \varepsilon_j$ and $\sighat_j^2 \leq \delta_j^2$ for some $\varepsilon_j, \delta_j \geq 0$ and all $j \in \{1, \ldots, m\}$.
  Let $\varepsilon_\textup{max} = \max\{\varepsilon_1, \ldots, \varepsilon_m\}$ and $\delta_\textup{max} = \max\{\delta_1, \ldots, \delta_m\}$.
  Then
  \begin{equation}
    \lvert I - \Ihat_{\textup{p}} \rvert = O(\varepsilon_\textup{max}) + O(\delta_\textup{max}^2) \: \text{ and } \: \sighat_{\textup{p}}^2 = O(\delta_\textup{max}^2) .
  \end{equation}  
  as $\varepsilon_\textup{max}, \delta_\textup{max} \to 0$.
\end{theorem}
\begin{proof}
  Expanding $(\bm{C} + \bm{\Sigma})^{-1} = (\bm{I} + \bm{C}^{-1} \bm{\Sigma})^{-1} \bm{C}^{-1}$ as a Neumann series gives
  \begin{equation}
    \Ihat_{\text{p}} = \bm{c}^\top \bm{C}^{-1} \hat{\bm{I}} + \sum_{k=1}^\infty \bm{c}^\top (- \bm{C}^{-1} \bm{\Sigma})^k \bm{C}^{-1} \hat{\bm{I}} .
  \end{equation}
  The series converges when $\delta_\textup{max}$ is sufficiently small since $\bm{\Sigma} = \operatorname{diag}(\sighat_1^2, \dots, \sighat_m^2)$.
  Eqs.\@~\eqref{eq:cross-covariance} and~\eqref{eq:cross-cov-total} yield $\bm{C}^{-1} \bm{c} = (1, \ldots, 1)^\top$.
  Therefore $\Ihat_{\text{p}} = \sum_{i=1}^m \Ihat_i + O(\delta_\textup{max}^2)$, so that  
  \begin{equation}
    \lvert I - \Ihat_{\text{p}} \rvert \leq \sum_{j=1}^m \lvert I_j - \Ihat_j \rvert + O(\delta_\textup{max}^2),
  \end{equation}
  which gives the claim for $\lvert I - \Ihat_{\text{p}} \rvert$.
  The proof for $\sighat_{\text{p}}^2$ is analogous.
\end{proof}

Convergence and variance contraction rates for BQ are well-understood when the kernel is a stationary Matérn kernel of order $\nu$~\citep{Kanagawa2019, Kanagawa2020, mahsereci2026}.
Here we consider stationary Matérns rather than tensor product Matérns because convergence results for the latter are more complicated~\citep[Sec.\@~5.1.2]{mahsereci2026}.
Consider the posterior for standard BQ given in Eqs.\@~\eqref{eq:bq-posterior}--\eqref{eq:bq-var} and let $\eta > 0$.
If $\Omega$ is bounded and its boundary is sufficiently regular (e.g., $\Omega$ is a rectangle) and $f$ is an element of $H^{\eta + d/2}(\Omega)$, the Sobolev space of order $\eta + d/2$ on $\Omega$, one can prove that
\begin{equation}
  \lvert I - \Ihat \rvert = O(N^{-\min\{\eta, \nu\}/d - 1/2})
\end{equation}
and $\sighat^2 = O(N^{-2\nu/d - 1})$ provided that the locations $\bm{X} = ( \bm{x}_1, \dots, \bm{x}_N )$ are quasi-uniform~\citep[Thm.\@~5.6]{mahsereci2026}.
Quasi-uniformity means that the separation radius and fill-distance,
\begin{equation*}
  q_N = \frac{1}{2} \min_{1 \leq i \neq j \leq N} \lVert \bm{x}_i - \bm{x}_j \rVert \: \text{ and } \: h_N = \sup_{\bm{x} \in \Omega} \min_{\bm{x}_i \in \bm{X}} \lVert \bm{x} - \bm{x}_i \rVert,
\end{equation*}
satisfy $q_N = \Theta(h_N)$ as $N \to \infty$.
Separation radius measures the smallest distance between any two integration locations while the fill-distance is the radius of the largest ball in $\Omega$ that contains none of $\bm{x}_i$.
Applying these asymptotics to the children gives the following corollary of \cref{thm:tree-error}.

\begin{corollary}
  Suppose that each $\Omega_j \subset \mathbb{R}^d$ is a rectangle and that a stationary Matérn kernel of order $\nu_j$ and $N_j$ quasi-uniform observation locations are used to obtain~\eqref{eq:leaf-integral-estimate}.
  If $f|_{\Omega_j} \in H^{\eta_j + d/2}(\Omega_j)$ for $j \in \{1, \ldots, m\}$, then
  \begin{equation*}    
    \lvert I - \Ihat_{\textup{p}} \rvert = O\big( \max\{ N_j^{-\min\{\eta_j, \nu_j\}/d - 1/2} \}_{j=1}^m \big)
  \end{equation*}
  and 
  \begin{equation*}
    \sighat_\textup{p}^2 = O\big( \max\{ N_j^{-2\nu_j/d - 1} \}_{j=1}^m \big) 
  \end{equation*}
  as $N_1, \ldots, N_m \to \infty$.
\end{corollary}
This corollary states that the convergence of HBQ is controlled by the smallest local smoothness of $f$ (the smallest $\eta_j$) or the roughest local Matérn model (the smallest $\nu_j$).

\section{Adaptive Tree Construction}\label{sec:adaptive-alg}
In practice, the hierarchy in \cref{sec:method} is not known a priori and must be determined from the integrand evaluations.
We grow the partition adaptively: posterior uncertainty guides us where to invest more integrand evaluations, while model selection techniques decide when to partition further.

\subsection{Splitting Criterion}\label{sec:splitting}
We use model selection to decide when a single stationary GP is misspecified and should be replaced by GPs on smaller subdomains.

\textbf{When to split?}
Let $f$ be a GP on the domain $\domain = [a, b]$ conditioned on evaluation points $\bm{X}$.
Select a point $x_{\text{split}} \in \bm{X} \setminus \{ a, b \}$.
We obtain two independent GPs: $f_1$ on the domain $\Omega_1 = [a, x_{\text{split}}]$ conditioned on evaluation points $\bm{X}_1 = \bm{X} \cap [a, x_{\text{split}}]$, and $f_2$ on the domain $\Omega_2 = [x_{\text{split}}, b]$ conditioned on evaluation points $\bm{X}_2 = \bm{X} \cap [x_{\text{split}}, b]$.
We can evaluate the Bayes factor of the two models:
\begin{equation}
  \text{BF}_{\text{split}} = \frac{p(\bm{y} \mid f_1, f_2)}{p(\bm{y} \mid f)} = \frac{p(\bm{y}_1 \mid f_1) p(\bm{y}_2 \mid f_2)}{p(\bm{y} \mid f)}
\end{equation}
Since the Bayes factor requires intractable marginalization over hyperparameters, we approximate it with the Bayesian Information Criterion (BIC) \citep{schwarz1978}, which penalizes model complexity using the profile marginal likelihood from \cref{sec:method-hyperparams}: 
Let $\mathcal{L}_{\text{parent}}$ and $\mathcal{L}_{\text{split}}$ denote the maximized profile log-likelihoods of the single-GP and split models, respectively.
We accept a split if
\begin{equation}
  \mathcal{L}_{\text{split}} - \mathcal{L}_{\text{parent}} > \frac{\Delta k}{2} \log N,
  \label{eq:bic-criterion}
\end{equation}
where $N$ is the number of evaluation points and $\Delta k$ is the number of additional parameters introduced by the split.

With a shared lengthscale, the single-GP model has parameters $(\outscale, \lscale)$; the split model has two such pairs plus the split location, giving $\Delta k = 3$.
With per-dimension lengthscales (ARD), $\Delta k = d + 2$.
The penalty grows with $N$, preventing over-fragmentation as the evaluation budget increases.

To further guard against pathological splits, we require each child to contain at least $5^d$ evaluation points.
This prevents the creation of leaves with too few points for reliable hyperparameter estimation.

\textbf{Where to split?}
Given the criterion in \cref{eq:bic-criterion}, it remains to select the split location.
We restrict candidates to existing evaluation points and perform a grid search along each coordinate axis.
For each candidate, we fit local hyperparameters to both children and evaluate the combined log-likelihood $\mathcal{L}_{\text{split}}$.
We select the axis and location that maximize $\mathcal{L}_{\text{split}}$, accepting the split only if it satisfies \cref{eq:bic-criterion}.

\subsection{Hyperparameter Fitting}\label{sec:method-hyperparams}
Each leaf's kernel~\eqref{eq:kernel} has hyperparameters $(\outscale, \lscale, \nu)$.
The output scale $\outscale$ is concentrated out analytically, leaving $\lscale$ as the sole continuous parameter to optimize via the profile marginal likelihood.
We select the smoothness $\nu$ by a discrete search over a set of candidate Matérn orders, optimizing $\lscale$ for each and retaining the best pair.
This allows different subdomains to use different kernel smoothness, matching local regularity.

Adaptive splitting means that leaves frequently have few evaluation points, causing the profile likelihood for $\lscale$ to plateau near zero.
We regularize with a log-normal prior on $\lscale$ anchored at the parent's fitted lengthscale:
\begin{equation}
  \log \lscale_{\text{child}} \sim \Normal(\log \bar{\lscale}_{\text{parent}},\, \sigma^2_{\text{prior}}),
  \label{eq:child-lscale-prior}
\end{equation}
where $\bar{\lscale}$ is the geometric mean of the parent's per-dimension lengthscales in the ARD setting.
Hyperparameters are re-estimated whenever a leaf is refined.
\cref{app:hyperparams} gives further details, including the root node prior and candidate sets.

\subsection{Adaptive Refinement}\label{sec:adaptive}
Adaptively growing the partition is a sequential decision making problem:
At each step, we can take one of two actions --- split a leaf to reduce model misspecification, or add a batch of integrand evaluations to a leaf.

\textbf{Integral variance reduction.}
Intuitively, at each step we want to act on the leaf that would most improve our integral estimate.
Let $\sighat^2$ denote the current posterior variance of the global integral estimate, and let $\sighat^2_{k+}$ denote the hypothetical variance obtained by adding a batch of $M_k$ integrand evaluations to leaf $k$.
Here $M_k$ may depend on the leaf --- for instance, inserting midpoints into an existing grid of $N_k$ points adds $N_k - 1$ new evaluations.
We score each leaf by its expected variance reduction per evaluation,
\begin{equation}
  s_k = \frac{\sighat^2 - \sighat^2_{k+}}{M_k},
  \label{eq:leaf-score}
\end{equation}
and select the highest-scoring leaf $k^* = \arg\max_k s_k$.
This is a batched variant of the integral variance reduction (IVR) criterion used in active BQ \citep{osborne2012}, extended to the hierarchical setting.
The hypothetical variance $\sighat^2_{k+}$ is computed cheaply by re-evaluating the BQ posterior variance formula \cref{eq:bq-var} under the current kernel with the augmented point set, without re-fitting hyperparameters or evaluating $f$.
The cost normalization by $M_k$ prevents bias toward leaves with large grids.

\textbf{Split or refine.}
Given target leaf $k^*$, we attempt a BIC-guided split (\cref{sec:splitting}) and recursively try to split the resulting children until no further splits are accepted. If the initial split is rejected, we refine by adding evaluation points and re-fitting hyperparameters.
The loop repeats until the budget is exhausted or $\sighat$ falls below a tolerance (\cref{alg:hbq}).

\begin{algorithm}[t]
  \caption{Hierarchical Bayesian Quadrature}\label{alg:hbq}
  \begin{algorithmic}[1]
    \REQUIRE integrand $f$, domain $\domain$, budget $N_{\max}$, tolerance~$\tau$
    \ENSURE integral estimate $\Ihat \pm \sighat$
    \STATE Initialize root node with GP fitted on initial evaluations
    \STATE $\Ihat, \sighat \leftarrow \textsc{BQ}(\text{root})$
    \WHILE{$N_{\text{used}} < N_{\max}$ \AND $\sighat > \tau$}
      \STATE Select leaf $k^*$ with largest expected variance reduction
      \IF{\textsc{ShouldSplit}$(k^*)$}
        \STATE Split $k^*$: create children, fit local GPs
        \STATE Recursively split children while accepted
      \ELSE
        \STATE Refine $k^*$: add evaluation points, refit hyperparameters
      \ENDIF
      \STATE $\Ihat, \sighat \leftarrow \textsc{Recombine}(\text{root})$
    \ENDWHILE
  \end{algorithmic}
\end{algorithm}

\section{Related Work}\label{sec:related}

\textbf{Partitioned Gaussian processes.}
Fitting local stationary GPs on subdomains is a well-established strategy for handling nonstationarity in regression.
Treed GP models \citep{gramacy2008} recursively partition the input space and fit independent GPs per leaf, with the tree structure learned via reversible-jump MCMC.
Because the leaf GPs are independent, predictions at subdomain boundaries are discontinuous.
Several aggregation schemes address this, e.g.\ product-of-experts \citep{tresp2000,deisenroth2015} and boundary pseudo-observations \citep{park2018}.
Nested kriging \citep{rulliere2018} aggregates local predictions through a second-level GP --- the regression analogue of our tree conditioning.
Most related to our splitting criterion, \citet{hinne2022bndd} also use BIC to choose between a merged GP and a split GP with a discontinuity, albeit for regression and discontinuity design rather than integration.
None of these approaches have been applied to integration, where the recombination mechanism directly affects integral uncertainty and where kernel mean embeddings impose additional structure.

\textbf{Local Bayesian optimization.}
Trust-region methods such as TuRBO~\citep{eriksson2019turbo} also fit local stationary GPs, but for optimization: acquisition functions bias evaluations toward extrema and abandon regions of low expected improvement.
For quadrature every subdomain contributes additively---small-amplitude regions still matter through their volume---so HBQ instead keeps a global partition and recombines the local integral estimates into a single posterior over the integral, with no analogue of TuRBO's trust-region shrinkage.
The BQ-specific contribution is thus not local stationary models per se, but their combination with closed-form subdomain integration and hierarchical recombination of integral functionals.

\textbf{Bayesian quadrature.}
Bayesian quadrature \citep{ohagan1991,mahsereci2026} places a GP prior on the integrand and derives a posterior distribution over the integral.
Active point selection via integral variance reduction \citep{osborne2012}, Frank--Wolfe optimization \citep{briol2015}, and warped models for non-negative integrands \citep{gunter2014} have substantially improved practical performance.
All of these methods use a single global GP, which limits their ability to adapt to spatially varying integrand complexity.
BART-Int \citep{zhu2020} replaces the GP with a Bayesian additive regression tree prior, whose sum-of-trees structure implicitly partitions the domain.
However, BART-Int's piecewise-constant tree basis limits its effectiveness for smooth integrands, where GP-based BQ retains an advantage.
Multilevel BQ \citep{li2023} introduces hierarchy, but across fidelity levels rather than spatial subdomains.
Our approach combines domain partitioning with GP-based BQ in each subdomain, preserving the smoothness advantages of GP priors while adapting to nonstationarity through the tree structure.

\textbf{Adaptive cubature.}
The Genz--Malik algorithm \citep{genz1980,berntsen1991} is the gold standard adaptive cubature method.
Like HBQ, it partitions the domain into axis-aligned hyperrectangles via binary splits, concentrating effort where the integrand is most difficult.
HBQ can be viewed as its Bayesian counterpart: we replace Genz--Malik's embedded polynomial rules with local GP surrogates, gaining principled uncertainty quantification, and split only when the local model is detectably misspecified (\cref{sec:splitting}), refining the existing GP otherwise.

\textbf{Model evidence computation.}
When the integral is a model evidence, sampling-based estimators such as Warp-III bridge sampling~\citep{gronau2019warp3} and the truncated harmonic-mean estimator THAMES~\citep{metodiev2025thames} are also widely used.
These require posterior samples rather than pointwise black-box evaluations and do not quantify integral uncertainty, but they are effective complementary tools for model comparison.

\begin{figure*}[t]
  \centering
  \includegraphics[width=1.0\textwidth]{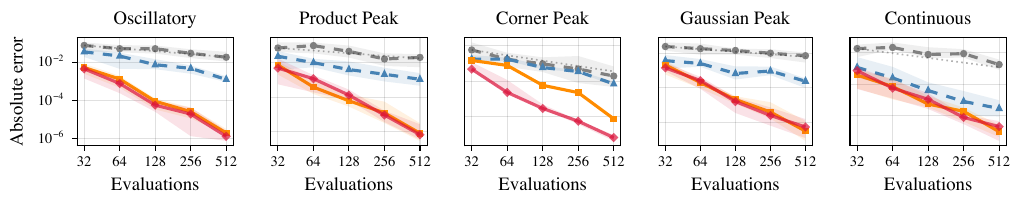}
  \caption{
    \textbf{Convergence on the Genz test function suite} ($d = 2$, difficulty $C = 9$, $20$ trials).
    Lines show median absolute error; shaded regions indicate the interquartile range.
    The dotted line shows the $N^{-1/2}$ Monte Carlo reference rate.
    \textcolor[HTML]{DC143C}{HBQ} matches \textcolor[HTML]{FF8C00}{BQ} on four families and substantially outperforms it on Corner Peak, where spatially varying complexity rewards adaptive partitioning.
    \textcolor[HTML]{808080}{MC} and \textcolor[HTML]{4682B4}{QMC} shown for reference.
  }
  \label{fig:genz}
\end{figure*}

\section{Experiments}\label{sec:experiments}
We provide a Julia implementation of our method.\footnote{\url{https://github.com/timweiland/HierarchicalBayesQuad.jl}}
Refer to \cref{app:experiments} for further details on all experiments.

\subsection{Genz Benchmark Functions}\label{sec:genz}

We evaluate HBQ on the standard Genz test function suite \citep{genz1984}, a collection of parameterized integrand families over $[0,1]^d$, each isolating a specific challenge for quadrature methods. We use five of the six families: Oscillatory (high-frequency oscillations), Product Peak (localized peak with product structure), Corner Peak (concentration near the origin), Gaussian Peak (localized Gaussian bump), and Continuous ($C^0$ kink).
Each family is parameterized by vectors $\bm{a}, \bm{u} \in \mathbb{R}^d$ controlling difficulty and location, with analytic integrals available for all families.

\textbf{Setup.}
We test in $d = 2$ with difficulty $C = \sum_i a_i = 9$, drawing $\bm{a}$ and $\bm{u}$ randomly for each of $20$ independent trials.
We compare four methods at evaluation budgets $N \in \{32, 64, 128, 256, 512\}$:
Monte Carlo (MC), quasi-Monte Carlo with Sobol sequences (QMC), standard BQ with a single stationary GP, and HBQ.
Both BQ and HBQ select the kernel smoothness automatically from $\{\text{Mat\'ern-}1/2, \text{Mat\'ern-}3/2, \text{Mat\'ern-}5/2\}$ via the profile marginal likelihood, ensuring a fair comparison.
BQ uses Sobol points with no adaptive refinement; HBQ uses the full adaptive loop of \cref{alg:hbq}.

\textbf{Results.}
\Cref{fig:genz} shows median absolute error with interquartile bands.
HBQ substantially outperforms BQ on Corner Peak, where the integrand concentrates sharply near the origin and flattens elsewhere, achieving up to $6\times$ lower error.
HBQ's adaptive partitioning concentrates evaluations in the region of high complexity rather than spreading them uniformly.
On the remaining four families, HBQ performs comparably to BQ, confirming that the adaptive overhead does not degrade performance when the integrand is well-described by a single stationary GP.

\subsection{Model Evidence for an Epidemiological Model}\label{sec:sir}

We return to the motivating example from \cref{fig:sir-motivation}: computing the model evidence $Z = \int p(y \mid \beta, \gamma)\, p(\beta, \gamma)\, d\beta\, d\gamma$ for a susceptible-infected-recovered (SIR) model with two parameters, infection rate $\beta$ and recovery rate $\gamma$.
Each likelihood evaluation requires solving an ODE system, making function evaluations expensive.
We use a small-population setting ($N_{\text{pop}} = 200$) with two weekly incidence observations under Poisson noise, producing a broad but sharply curved ridge in the likelihood surface (details in the appendix).

\Cref{fig:sir-motivation}(c) compares convergence of four methods using a ground truth computed via adaptive cubature.
HBQ with independent summation achieves sub-1\% relative error at $N = 100$ evaluations, while BQ, MC, and QMC remain above 1\%.
Panel~(b) shows how HBQ concentrates its evaluation budget: at $N = 150$, the tree allocates five leaves with the densest coverage along the likelihood ridge, while the flat regions receive only coarse coverage.

\subsection{Structure Recovery and Calibration}\label{sec:calibration}

We now test whether HBQ can recover known nonstationary structure from data.
We sample integrands from a piecewise Mat\'ern GP on $[0,1]$ with three segments---rough edges (Mat\'ern-$1/2$, $\lscale = 0.02$) flanking a smooth center (Mat\'ern-$5/2$, $\lscale = 0.1$)---so that the true partition, kernel smoothness, and integral are all known exactly.
Since the integrand comes from a piecewise GP, this is an on-model test: any errors are due to the method, not model misspecification.
We run 200 trials with $N = 400$ evaluations each.

\textbf{Structure recovery.}
HBQ recovers the correct number of segments in 100\% of trials and selects the correct kernel orders in 96\%.
\Cref{fig:calibration-fits} shows a representative trial: BQ selects the roughest available kernel to accommodate the rough edges, but cannot adapt to local smoothness, leaving it overly uncertain in the smooth center.
HBQ identifies the piecewise structure and fits each region with an appropriate kernel, yielding $3.3{\times}$ tighter uncertainty.

\textbf{Calibration.}
\Cref{fig:calibration-qq} shows that this tighter uncertainty is earned: $z$-scores for both methods track $\mathcal{N}(0,1)$ well (BQ: 92\% coverage of a nominal 95\% CI, $\sigma_z = 1.13$; HBQ: 90\%, $\sigma_z = 1.53$).\footnote{One of the 200 BQ trials produced a degenerate posterior with $\hat{\sigma} = 0$, i.e.\ an infinite $z$-score, and is excluded from BQ's $\sigma_z$.}
HBQ has four outliers with $|z| > 4$, all caused by imprecise split placement at the sharp smoothness transitions in this synthetic problem; when splits land correctly, HBQ is well calibrated ($\sigma_z \approx 1$).

\begin{figure}[t]
  \centering
  \includegraphics[width=\columnwidth]{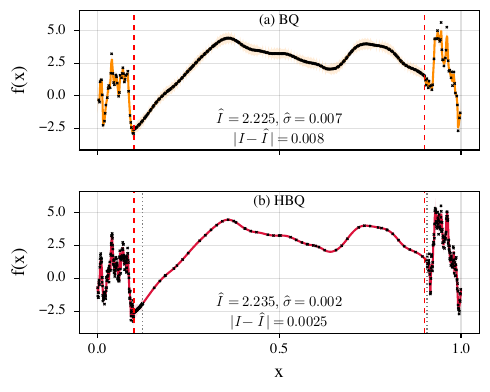}
  \caption{
    \textbf{GP posterior fits on a piecewise integrand.}
    The integrand is sampled from a piecewise Mat\'ern GP with narrow rough edges (Mat\'ern-$1/2$, $\lscale = 0.02$) and a wide smooth center (Mat\'ern-$5/2$, $\lscale = 0.1$); true boundaries are shown as dashed red lines.
    \textbf{(a)}~BQ fits a single stationary GP.
    \textbf{(b)}~HBQ adaptively partitions the domain and recovers the piecewise structure (dotted gray lines), achieving $3.3{\times}$ tighter uncertainty.
  }
  \label{fig:calibration-fits}
\end{figure}

\begin{figure}[t]
  \centering
  \includegraphics[width=\columnwidth]{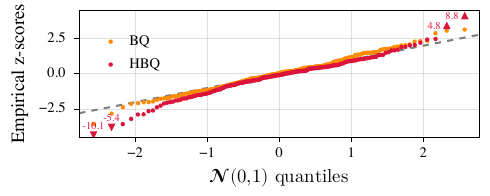}
  \caption{
    \textbf{QQ plot of $z$-scores} against $\mathcal{N}(0,1)$ over 200 trials on the piecewise integrand.
    Both BQ (orange) and HBQ (red) track the diagonal well.
    HBQ has four outliers (triangles) attributable to imprecise split placement at sharp smoothness transitions.
  }
  \label{fig:calibration-qq}
\end{figure}

\subsection{Reaction--Diffusion PDE}\label{sec:pde}

We test HBQ on integrands arising from a one-dimensional reaction--diffusion PDE.
The diffusivity follows a sigmoid step, modelling a composite material with a smooth, slowly varying region and a spiky, rapidly varying region.
Random Gaussian bump sources and point sources produce sharp local features in the low-diffusivity region.
This creates integrands with genuine multi-scale character: a single stationary GP cannot simultaneously capture the smooth and spiky regimes without either oversmoothing or overfitting.

\textbf{Setup.}
We draw 20 independent PDE instances with randomized source configurations and compare MC, QMC, BQ, and HBQ at budgets $N \in \{8, 16, 32, 64, 128, 256, 512\}$.
Ground-truth integrals are computed via the trapezoidal rule on a fine grid ($N_{\text{fine}} = 8000$).

\textbf{Results.}
\Cref{fig:pde} shows median absolute error with interquartile bands.
All methods perform similarly at small budgets ($N \leq 32$), but HBQ pulls away sharply from $N = 64$ onward.
At $N = 512$, HBQ achieves a median error of $6.7 \times 10^{-6}$, roughly $35{\times}$ lower than BQ ($2.3 \times 10^{-4}$).
HBQ automatically allocates different lengthscales and point densities to the smooth and spiky regions.

\begin{figure}[t]
  \centering
  \includegraphics[width=\columnwidth]{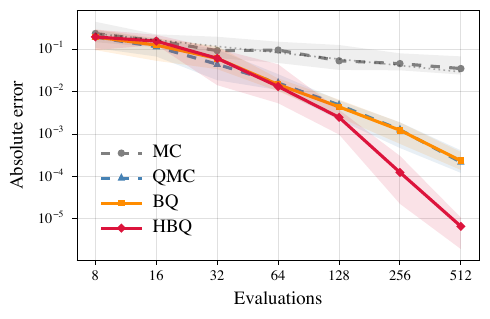}
  \caption{
    \textbf{Convergence on 1D reaction--diffusion PDE integrands.}
    HBQ separates from BQ at $N = 64$ and reaches $35{\times}$ lower error at $N = 512$.
  }
  \label{fig:pde}
\end{figure}

\section{Discussion and Conclusion}\label{sec:conclusion}

We presented Hierarchical Bayesian Quadrature, a method that addresses nonstationarity in BQ by adaptively partitioning the domain into subdomains with local stationary GPs.
A hierarchy of GP conditioning recombines the local integral estimates while preserving cross-subdomain correlations.
Model selection via BIC controls tree growth, preventing unnecessary partitioning: on stationary integrands, HBQ recovers the performance of standard BQ.
On nonstationary integrands, HBQ achieves substantial gains by concentrating evaluations where the integrand is most complex.

Our convergence analysis shows that the global error is controlled by the worst-case local smoothness (\cref{thm:tree-error}), providing a theoretical foundation for the adaptive strategy.
The structure recovery experiment confirms that HBQ can identify piecewise-stationary structure from data and produce well-calibrated uncertainty estimates.

HBQ is closely related to Genz--Malik adaptive cubature: both partition the domain into axis-aligned hyperrectangles and adaptively subdivide.
HBQ replaces polynomial quadrature rules with GP surrogates, gaining principled uncertainty quantification and data-driven split decisions.

\subsection{Scope and Limitations}\label{sec:limitations}
HBQ targets expensive, low-to-moderate-dimensional integration whose integrand has spatially varying behavior---for example model evidences over peaked or ridged likelihoods, or simulator and PDE responses with multi-scale features.
This is the regime where quadrature generally outperforms Monte Carlo, and where adaptive subdivision has been the standard response to nonstationarity since Genz--Malik.
Because splitting is BIC-gated, HBQ reduces to standard BQ when a single stationary model suffices, so its gains concentrate on genuinely nonstationary integrands---as the flat performance on four of the five Genz families shows.

It extends to higher dimensions through tensor-product kernels and the $d$-dependent penalty, and a $d = 3$ benchmark (\cref{app:higherd}) confirms the advantage persists beyond our two-dimensional experiments, though very high dimensions remain hard as axis-aligned partitioning and the $5^d$ per-leaf requirement scale unfavourably.
The main structural limitation is the restriction to axis-aligned hyperrectangles and closed-form-embedding kernels, which is what keeps recombination closed-form; relaxing it would require numerical mean-embedding approximation.

Finally, the convergence analysis is deliberately partial---conditional on the final tree and kernels, not yet covering tree construction or hyperparameter adaptation---which we leave, together with more flexible partitions, to future work.

\begin{contributions} 
    Tim Weiland developed the method, implemented the software, designed and ran the experiments, and led the writing.
    Toni Karvonen proposed the initial project idea, contributed the convergence analysis and its theoretical framing, and co-supervised the work.
    Philipp Hennig supervised the project and contributed to its conceptual development and presentation.
\end{contributions}

\begin{acknowledgements} 
TW and PH gratefully acknowledge co-funding by the European Union (ERC, ANUBIS, 101123955).
Views and opinions expressed are however those of the author(s) only and do not necessarily reflect those of the European Union or the European Research Council.
Neither the European Union nor the granting authority can be held responsible for them.
PH is supported by the DFG through Project HE 7114/6-1 in SPP2298/2.
PH is a member of the Machine Learning Cluster of Excellence, funded by the Deutsche Forschungsgemeinschaft (DFG, German Research Foundation) under Germany’s Excellence Strategy – EXC number 2064/1 – Project number 390727645.
TW and PH also gratefully acknowledge the German Federal Ministry of Education and Research (BMBF) through the Tübingen AI Center (FKZ:01IS18039A); and funds from the Ministry of Science, Research and Arts of the State of Baden-Württemberg.
The authors thank the International Max Planck Research School for Intelligent Systems (IMPRS-IS) for supporting TW.
TK was supported by Research Council of Finland projects 359183 (``Flagship of Advanced Mathematics for Sensing, Imaging and Modelling'') and 368086 (``Inference and approximation under misspecification''), and acknowledges the research environment provided by ELLIS Institute Finland.
This research was supported by the Research Council of Finland and the German Academic Exchange Service (DAAD, Projekt 57763386, using funds of the BMFTR) in the form of the mobility project ``Probabilistic numerics under non-Stationarity''.
\end{acknowledgements}

\bibliography{main}

\newpage

\onecolumn

\title{Hierarchical Bayesian Quadrature\\(Supplementary Material)}
\maketitle

\appendix

\section{Experimental Details}\label{app:experiments}

\subsection{Genz Benchmark Functions}\label{app:genz}

The Genz test suite \citep{genz1984} consists of six parameterized integrand families over $[0,1]^d$.
We use the five families that are continuous (omitting Discontinuous, which requires $d \geq 2$ and has a jump discontinuity):
\begin{align*}
  \text{Oscillatory:}    &\quad f(\xvec) = \cos\bigl(2\pi u_1 + \textstyle\sum_i a_i x_i\bigr), \\
  \text{Product Peak:}   &\quad f(\xvec) = \textstyle\prod_i \bigl(a_i^{-2} + (x_i - u_i)^2\bigr)^{-1}, \\
  \text{Corner Peak:}    &\quad f(\xvec) = \bigl(1 + \textstyle\sum_i a_i x_i\bigr)^{-(d+1)}, \\
  \text{Gaussian Peak:}  &\quad f(\xvec) = \exp\bigl(-\textstyle\sum_i a_i^2 (x_i - u_i)^2\bigr), \\
  \text{Continuous:}     &\quad f(\xvec) = \exp\bigl(-\textstyle\sum_i a_i |x_i - u_i|\bigr).
\end{align*}
The vectors $\bm{a} \in \R^d$ and $\bm{u} \in [0,1]^d$ control difficulty and location, respectively.
Analytic integrals are available for all families.

\textbf{Parameter generation.}
For each of $20$ trials, we sample $\bm{u} \sim \mathrm{Unif}([0,1]^d)$ and draw $\bm{a}$ by sampling a weight vector $\bm{w} \sim \mathrm{Unif}([0,1]^d)$ and setting $a_i = C \cdot w_i / \sum_j w_j$, so that $\sum_i a_i = C$.
We use difficulty $C = 9$.
This randomization over parameters (rather than fixing $\bm{a}$ and $\bm{u}$) tests robustness across a range of integrand configurations.

\textbf{Method configuration.}
We test in $d = 2$ with evaluation budgets $N \in \{32, 64, 128, 256, 512\}$.
\begin{itemize}
  \item \textbf{MC}: uniform random points, independent random seed per trial/budget.
  \item \textbf{QMC}: deterministic Sobol sequence (same points across trials for a given $N$).
  \item \textbf{BQ}: single stationary GP over $[0,1]^2$ with Sobol initial points. Kernel smoothness selected from $\{\text{Mat\'ern-}1/2, 3/2, 5/2\}$ by profile marginal likelihood. No adaptive refinement (points fixed at initialization).
  \item \textbf{HBQ}: full adaptive loop (\cref{alg:hbq}) with $5$ initial points per dimension ($25$ total), Sobol refinement adding $20$ points per step, tree conditioning, BIC splitting criterion with midpoint splits, and the same kernel candidates. The ARD lengthscale mode is used, with a log-normal child prior anchored at the parent's fitted lengthscale.
\end{itemize}
Error is measured as $|I - \hat{I}|$ using the analytic ground truth.

\subsection{SIR Model Evidence}\label{app:sir}

\textbf{SIR model.}
We use the standard deterministic SIR (susceptible--infected--recovered) model:
\begin{align*}
  \frac{dS}{dt} &= -\frac{\beta S I}{N}, &
  \frac{dI}{dt} &= \frac{\beta S I}{N} - \gamma I, &
  \frac{dR}{dt} &= \gamma I,
\end{align*}
with population $N = 200$, initial infected $I_0 = 5$, and simulation horizon $T = 50$ days.
The parameters of interest are the infection rate $\beta \in [0.1, 1.0]$ and recovery rate $\gamma \in [0.01, 0.5]$, with uniform priors over these ranges.

\textbf{Synthetic data.}
We generate data at ground truth $(\beta, \gamma) = (0.4, 0.15)$, corresponding to a basic reproduction number $R_0 = 2.67$.
Weekly incidence (new infections per week) is observed at weeks 1 and 5.
Observations are drawn from $y_k \sim \text{Poisson}(\lambda_k)$, where $\lambda_k = S(7(k{-}1)) - S(7k)$ is the true weekly incidence computed from the ODE solution.

\textbf{Integrand.}
The model evidence is $Z = \int p(y \mid \beta, \gamma)\, p(\beta, \gamma)\, d\beta\, d\gamma$.
For numerical stability, we integrate the shifted integrand $g(\beta, \gamma) = \exp(\ell(\beta, \gamma) - \ell_{\max})$, where $\ell(\beta, \gamma) = \log p(y \mid \beta, \gamma)$ is the Poisson log-likelihood and $\ell_{\max}$ is the maximum over a $200 \times 200$ grid.
The evidence is recovered as $Z = g_0 \cdot e^{\ell_{\max}} / |\Omega|$, where $g_0 = \int g\, d\beta\, d\gamma$.

\textbf{Ground truth.}
We compute the reference integral using an implementation of the Genz--Malik method with relative tolerance $10^{-10}$, yielding $\log Z \approx -7.01$.

\textbf{Method configuration.}
All methods receive the same evaluation budgets $N \in \{25, 50, 75, 100\}$.
BQ and HBQ use Sobol initial points; HBQ selects the kernel automatically from $\{\text{Mat\'ern-}1/2, 3/2, 5/2\}$, while BQ uses a fixed Mat\'ern-$3/2$ kernel.
HBQ uses the independent sum recombination (\cref{eq:independent-sum}) and the BIC splitting criterion.
MC results show the median and interquartile range over 20 independent runs.
QMC uses a single deterministic Sobol sequence.

\subsection{Structure Recovery and Calibration}\label{app:calibration}

\textbf{Piecewise GP construction.}
We define a piecewise Mat\'ern GP on $[0,1]$ with three segments:
\begin{itemize}
  \item $[0.0, 0.1]$: Mat\'ern-$1/2$, $\sigma = 5$, $\lscale = 0.02$ (rough),
  \item $[0.1, 0.9]$: Mat\'ern-$5/2$, $\sigma = 1$, $\lscale = 0.1$ (smooth), with imposed integral $\int f = 2$,
  \item $[0.9, 1.0]$: Mat\'ern-$1/2$, $\sigma = 5$, $\lscale = 0.02$ (rough).
\end{itemize}
Narrow rough edges flanking a wide smooth center create a challenging nonstationary structure: the rough regions occupy only 20\% of the domain but contribute substantial variability.

\textbf{Sampling procedure.}
For each trial, we draw $N = 400$ Sobol points on $[0,1]$, partition them into segments, and sample function values from the piecewise GP.
Segments are sampled left to right: the first segment is drawn jointly with its integral from the GP prior; each subsequent segment conditions on the boundary value from its left neighbor to ensure continuity.
The middle segment additionally conditions on its imposed integral.
This yields a continuous function realization with known per-segment integrals, so the ground-truth integral is $I = \sum_j I_j$.

\textbf{Method configuration.}
BQ fits a single stationary GP over the full domain with automatic kernel selection from $\{\text{Mat\'ern-}1/2, 3/2, 5/2\}$.
HBQ uses $N = 400$ evaluations with \texttt{initdiv}${}= 10$ ($41$ initial evaluation points), up to 2 splits (matching the true number of boundaries), and the same kernel candidates.
We run 200 independent trials and compute $z$-scores $z_i = (I_i - \hat{I}_i) / \hat{\sigma}_i$.

\textbf{Outlier analysis.}
Four of 200 HBQ trials have $|z| > 4$.
All four recover the correct number of segments and kernel orders; the cause is imprecise split placement.
When splits land at $0.125$ and $0.875$ (outside the rough regions), HBQ achieves $\sigma_z = 0.96$ over $n = 132$ trials.
When a split lands at $0.09375$ (inside the rough region $[0, 0.1]$), the smooth leaf inherits a sliver of rough territory that its Mat\'ern-$5/2$ kernel cannot capture, producing underestimated uncertainty.
This synthetic problem has instantaneous smoothness transitions (Mat\'ern-$1/2$ to Mat\'ern-$5/2$); real-world functions typically exhibit gradual transitions, providing more tolerance for split placement.

\subsection{Reaction--Diffusion PDE}\label{app:pde}

\textbf{PDE formulation.}
We consider the steady-state reaction--diffusion equation on $[0,1]$:
\[
  -\frac{d}{dx}\!\left(D(x)\frac{du}{dx}\right) + r\,u(x) = s(x), \qquad u(0) = u(1) = 0,
\]
with constant reaction rate $r = 5$.
The diffusivity models a composite material with a sigmoid transition:
\[
  D(x) = \exp\bigl(D_0 - \Delta D \cdot \sigma(x; c, \kappa)\bigr),
\]
where $\sigma(x; c, \kappa) = (1 + e^{-\kappa(x - c)})^{-1}$, $D_0 = -3$, $\Delta D = 5$, $c = 0.45$, and $\kappa = 25$.

\textbf{Source terms.}
Each trial uses $n_s = 7$ Gaussian bump sources $s(x) = \sum_j h_j \exp(-(x - c_j)^2 / w^2)$ with width $w = 0.02$, centers jittered uniformly around equispaced base positions, and heights drawn from $\mathrm{Unif}(5, 11)$.
In addition, $n_\delta = 4$ point sources of amplitude $A \sim \mathrm{Unif}(0.2, 0.7)$ are placed randomly in the low-diffusivity region $[0.55, 0.95]$, discretized as $A / h$ on the nearest grid point.
These create sharp, spatially localized spikes in the solution that are difficult for a single stationary GP to capture.

\textbf{Discretization and ground truth.}
The PDE is solved on a uniform grid with $N_\text{fine} = 8{,}000$ points using second-order finite differences with arithmetic-mean diffusivity at cell interfaces.
The ground-truth integral $\int_0^1 u(x)\,dx$ is computed via the trapezoidal rule on this fine grid.
Function evaluations use piecewise linear interpolation of the fine-grid solution.

\textbf{Method configuration.}
All methods receive the same evaluation budgets $N \in \{8, 16, 32, 64, 128, 256, 512\}$.
BQ and HBQ use Sobol initial points with automatic kernel selection from $\{\text{Mat\'ern-}1/2, 3/2, 5/2\}$.
HBQ uses the tree-conditioning recombination (\cref{eq:tree-conditioning}) and the BIC splitting criterion.
MC uses uniform random points; QMC uses a deterministic Sobol sequence.
We run 20 independent trials, each with a different random PDE instance.

\textbf{Multi-scale character.}
The characteristic lengthscale of the PDE solution scales as $\ell \sim \sqrt{D/r}$.
In the smooth region ($x < 0.45$), $\ell \approx 0.1$; in the spiky region ($x > 0.45$), $\ell \approx 0.008$.
This ${\sim}12{\times}$ lengthscale ratio is the source of HBQ's advantage: a single GP must compromise between these scales, while HBQ partitions the domain, fits each region with an appropriate lengthscale and invests integrand evaluations accordingly.

\section{Hyperparameter Fitting Details}\label{app:hyperparams}

\textbf{Smoothness selection.}
We search over the half-integer Matérn orders $\nu \in \{1/2, 3/2, 5/2\}$, corresponding to $C^0$, $C^1$, and $C^2$ sample paths, respectively.
For each candidate, we optimize the lengthscale $\lscale$ via the profile marginal likelihood and retain the $(\nu, \lscale)$ pair with the lowest penalized negative log-likelihood.

\textbf{Root node prior.}
For the root node, we use a penalized complexity (PC) prior~\citep{fuglstad2019} on the lengthscale, taking the constant function ($\lscale \to \infty$) as the base model and calibrating so that $P(\lscale < |\domain| / 10) = 0.025$.

\textbf{Child node prior.}
When a leaf splits, each child receives the log-normal prior~\eqref{eq:child-lscale-prior} anchored at the parent's fitted lengthscale.
In the ARD setting, where each dimension has its own lengthscale $\lscale_j$, we use the geometric mean $\bar{\lscale} = (\prod_j \lscale_j)^{1/d}$ as the anchor.
This prevents the lengthscale from escaping to degenerate values, which is particularly important in the ARD setting where per-dimension lengthscales along flat directions are otherwise unconstrained.

\section{Further Experiments}\label{app:additional}

These experiments complement the evaluation of \cref{sec:experiments}: they examine HBQ in a higher dimension (\cref{app:higherd}), its wall-clock cost against accuracy (\cref{app:walltime}), the effect of the recombination scheme (\cref{app:tcis}), and the reliability of the BIC criterion across penalty strengths and budgets (\cref{app:bic}).
All use the same Julia implementation as the main experiments.

\subsection{Higher-Dimensional Genz Benchmark}\label{app:higherd}
This experiment extends the two-dimensional Genz benchmark of \cref{sec:genz} to three dimensions to test whether the hierarchical construction retains its advantage in higher dimensions.
We integrate two Genz families over the unit cube $[0,1]^3$: Corner Peak, the canonical nonstationary integrand on which hierarchical partitioning is expected to help, and Product Peak, a localized but separable control.
For each family we draw $10$ random instances, sampling the difficulty vector as $\bm{a} = C \bm{w} / \sum_j w_j$ with $C = 9$ and $\bm{w}, \bm{u} \sim \mathrm{Unif}([0,1]^3)$.
We compare the same four estimators as in the main text---MC, QMC, BQ (a single-domain surrogate on Sobol points), and HBQ (adaptive tree refinement run to the budget)---at budgets $N \in \{64, 128, 256, 512, 1024\}$, selecting in both BQ and HBQ among Matérn kernels of smoothness $\nu \in \{1/2, 3/2, 5/2\}$.
\Cref{fig:e1} reports the median absolute error over the $10$ instances with interquartile bands.
At $N = 1024$, HBQ attains a median error of $1.82 \times 10^{-5}$ versus $8.81 \times 10^{-5}$ for BQ on Corner Peak (a $4.8\times$ improvement) and $2.1 \times 10^{-4}$ versus $3.8 \times 10^{-4}$ on Product Peak ($1.8\times$), while MC and QMC are one to several orders of magnitude worse.
HBQ matches or exceeds BQ at every budget, confirming that its gains extend to $d = 3$ and that it carries no penalty when the hierarchical advantage is small.

\begin{figure}[t]
  \centering
  \includegraphics[width=0.85\textwidth]{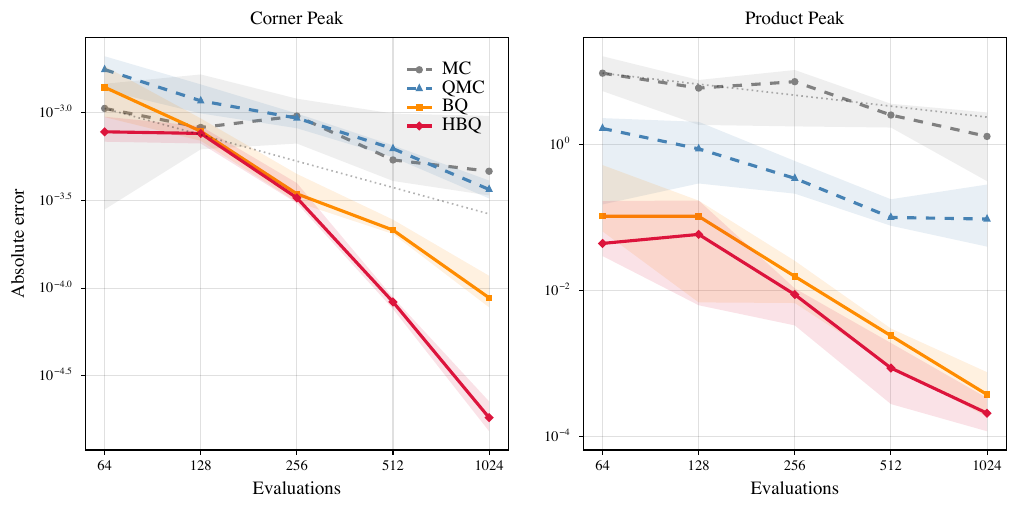}
  \caption{
    \textbf{Higher-dimensional Genz benchmark ($d = 3$).}
    Median absolute error (interquartile bands) versus the number of function evaluations on the Corner Peak and Product Peak families ($10$ random instances each).
    HBQ retains its advantage over BQ at $d = 3$, while MC and QMC are one to several orders of magnitude worse.
  }
  \label{fig:e1}
\end{figure}

\subsection{Wall-Clock Cost versus Accuracy}\label{app:walltime}
We measure wall-clock cost against accuracy on the Genz Corner Peak family in $d = 2$, using the parameter protocol and budget grid of \cref{sec:genz} ($C = 9$, $N \in \{32, 64, 128, 256, 512\}$).
Because the Genz integrand is analytically cheap, we inject a fixed $50\,\mathrm{ms}$ delay into every integrand evaluation to emulate an expensive simulator, so that total run time is dominated by the number of evaluations rather than by method bookkeeping.
Two implementation optimizations---inheriting the parent's selected kernel for split candidates, and a prescreen that fully refits only the winning split candidate---reduce HBQ's algorithmic overhead by up to roughly $20\times$ with no material change in accuracy (paired validation).
We compare MC, QMC, BQ, and HBQ on $10$ random instances, timing each run end-to-end after a warm-up pass and a garbage-collection call.
Because the injected cost dominates, all four methods incur essentially identical wall-clock at a given budget (about $13.6\,\mathrm{s}$ at $N = 256$ and $27.5\,\mathrm{s}$ at $N = 512$), so the comparison reduces to accuracy at matched compute (\cref{fig:e2}).
At $N = 256$, HBQ reaches a median absolute error of $5.1 \times 10^{-5}$ versus $4.2 \times 10^{-4}$ for BQ---about eightfold more accurate at matched wall-clock---and at $N = 512$, $2.0 \times 10^{-5}$ versus $7.9 \times 10^{-5}$, about fourfold; MC and QMC remain well above HBQ at moderate-to-large budgets.

\begin{figure}[t]
  \centering
  \includegraphics[width=0.7\textwidth]{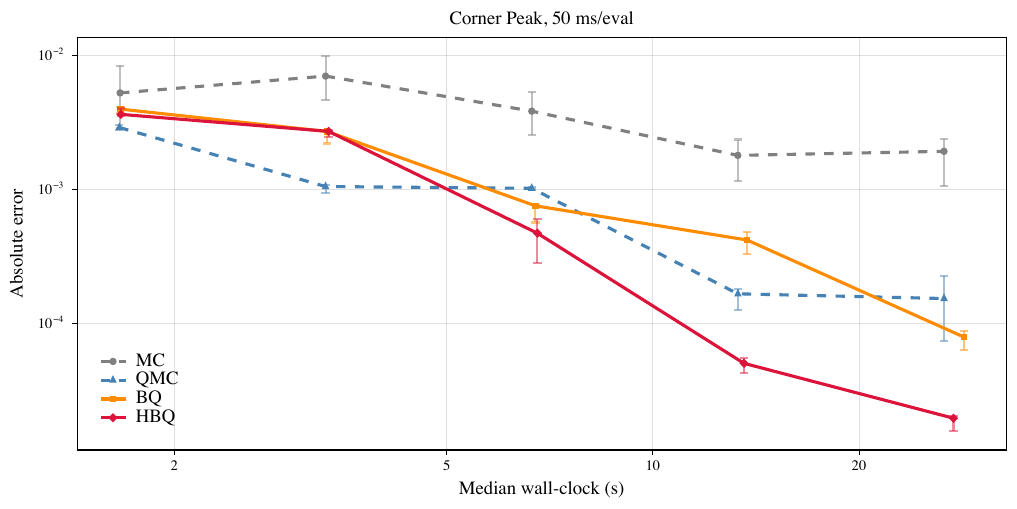}
  \caption{
    \textbf{Wall-clock cost versus accuracy.}
    Median absolute error against median wall-clock on Genz Corner Peak ($d = 2$, $10$ instances) with a $50\,\mathrm{ms}$ delay injected per evaluation.
    Because the evaluation cost dominates, all methods share essentially the same wall-clock at each budget, so HBQ's per-evaluation accuracy advantage translates directly into accuracy at matched compute.
  }
  \label{fig:e2}
\end{figure}

\subsection{Tree Conditioning versus Independent Sum}\label{app:tcis}
To quantify the effect of the recombination scheme, we re-ran the $d = 2$ Genz experiment under both the tree-conditioning and independent-sum modes with all other settings held fixed, so that the two differ only in the final recombination step (\cref{sec:recombination}).
We used Corner Peak---where HBQ shows its largest gain, and hence exercises the recombination machinery most---and Product Peak as a control whose tree rarely splits, at $C = 9$, drawing $20$ random instances per family and running each at budgets $N \in \{32, 64, 128, 256, 512\}$.
Comparing the two modes per instance and taking medians over trials, the ratio of median absolute error between tree conditioning and independent sum was $1.000$ in every family--budget cell (two families $\times$ five budgets; maximum deviation below $0.05\%$), and the posterior standard deviations matched to the same precision (\cref{fig:e3}).
On these families the cross-subdomain correlations reintroduced by tree conditioning therefore leave both the estimate and its reported uncertainty essentially unchanged, so the independent sum serves as an inexpensive ablation while tree conditioning remains the primary method at no accuracy cost.

\begin{figure}[t]
  \centering
  \includegraphics[width=0.85\textwidth]{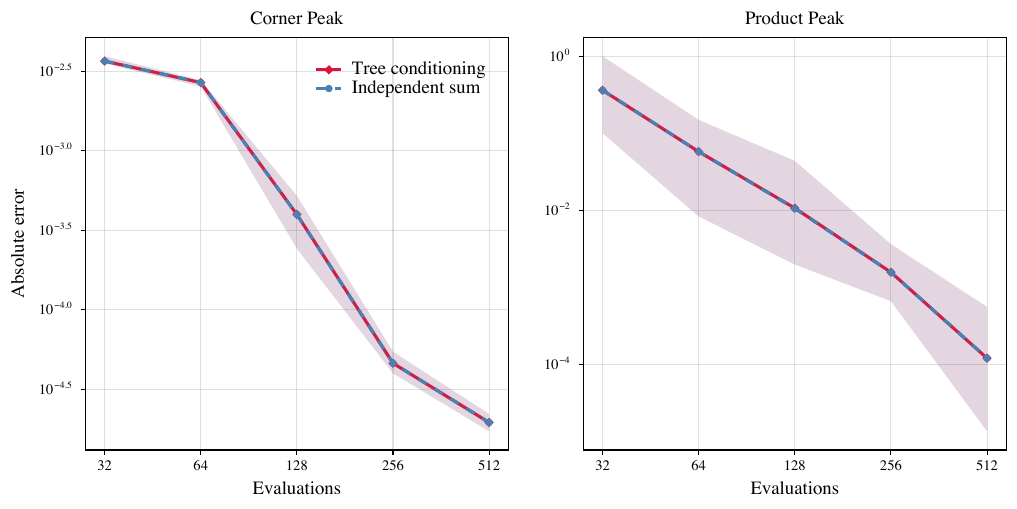}
  \caption{
    \textbf{Tree conditioning versus independent sum.}
    Convergence under the two recombination modes on Genz Corner Peak and Product Peak ($d = 2$, $20$ instances).
    The two modes coincide to within $0.05\%$ in both median absolute error and posterior standard deviation across all budgets.
  }
  \label{fig:e3}
\end{figure}

\subsection{BIC Threshold Sensitivity}\label{app:bic}
To assess whether the BIC criterion remains reliable at the small per-leaf sample sizes produced by adaptive splitting, we ran a sensitivity study on the piecewise-GP recovery task of \cref{sec:calibration}, whose ground truth comprises three segments---two narrow rough Matérn-$1/2$ edges flanking a wide smooth Matérn-$5/2$ center over $[0,1]$.
We swept the BIC penalty multiplier over $\{0.5, 1, 2, 4\}$, applied to the threshold of \cref{eq:bic-criterion} (with $1\times$ the default), and the evaluation budget over $N \in \{100, 200, 400\}$, running $100$ trials per cell---$1200$ in total---with trial seeds shared across multipliers for paired comparison.
For each cell we recorded the structure-recovery rate (fraction of trials recovering the correct number of leaves), the over- and under-split rates, the absolute integration error, empirical $95\%$ coverage, and the per-leaf evaluation counts (\cref{fig:e4}).
The criterion never over-split: the over-split rate was $0\%$ across all $1200$ trials, and at the default penalty recovery was $99\%$ at $N = 100$ and $100\%$ at $N = 200$ and $400$, even though the median per-leaf sample count was only about $34$ at $N = 100$.
The sole failure mode was conservative under-splitting under an aggressive penalty at the tightest budget ($40\%$ recovery, $60\%$ under-split at $4\times$ and $N = 100$), which resolved as the budget grew ($94\%$ at $N = 200$, $100\%$ at $N = 400$).
Empirical $95\%$ coverage stayed close to nominal ($87$--$96\%$) in all cells except the aggressive-penalty, tightest-budget cell ($4\times$, $N = 100$), where under-splitting reduced it to $75\%$.
In this regime the BIC approximation is thus empirically robust, erring if at all toward too few leaves rather than spurious partitions.

\begin{figure}[t]
  \centering
  \includegraphics[width=0.85\textwidth]{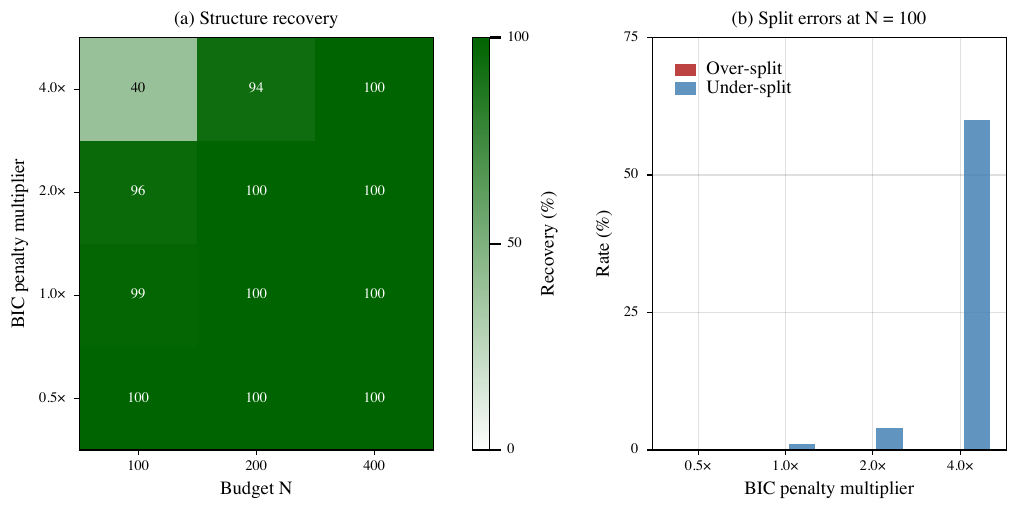}
  \caption{
    \textbf{BIC threshold sensitivity.}
    \textbf{(a)}~Structure-recovery rate over the BIC penalty multiplier and evaluation budget ($1200$ trials).
    \textbf{(b)}~Over- and under-split rates at $N = 100$.
    The over-split rate is $0\%$ throughout; the only failure mode is conservative under-splitting at the most aggressive penalty and tightest budget.
  }
  \label{fig:e4}
\end{figure}

\end{document}